\documentclass{article}
\usepackage{fullpage}
\usepackage{microtype}
\usepackage{graphicx}
\usepackage{subfigure}
\usepackage{booktabs}
\usepackage{hyperref}
\usepackage{mathtools,amsfonts,amsthm,amssymb}
\usepackage{comment}
\usepackage{indentfirst}
\usepackage[ruled]{algorithm2e}
\usepackage[numbers]{natbib}

\newcommand{\A}{\mathcal{A}}
\newcommand{\Cand}{\mathsf{Candidate}}
\newcommand{\D}{\mathcal{D}}
\newcommand{\eps}{\epsilon}
\newcommand{\err}{\mathrm{err}}
\newcommand{\F}{\mathcal{F}}
\newcommand{\G}{\mathcal{G}}
\newcommand{\Ind}[1]{\mathbb{I}\left[#1\right]}
\newcommand{\N}{\mathbb{N}}
\renewcommand{\O}{\mathcal{O}}
\newcommand{\Test}{\mathsf{Test}}
\newcommand{\X}{\mathcal{X}}

\newtheorem{theorem}{Theorem}[section]
\newtheorem{lemma}[theorem]{Lemma}
\newtheorem{claim}[theorem]{Claim}
\newtheorem{conjecture}[theorem]{Conjecture}
\newtheorem{proposition}[theorem]{Proposition}
\newtheorem{definition}[theorem]{Definition}

\begin{document}

\title{Do Outliers Ruin Collaboration?}
\author{Mingda Qiao\\Institute for Interdisciplinary Information Sciences (IIIS)\\Tsinghua University}
\date{}

\maketitle

\begin{abstract}
We consider the problem of learning a binary classifier from $n$ different data sources, among which at most an $\eta$ fraction are adversarial. The \emph{overhead} is defined as the ratio between the sample complexity of learning in this setting and that of learning the same hypothesis class on a single data distribution. We present an algorithm that achieves an $O(\eta n + \ln n)$ overhead, which is proved to be worst-case optimal. We also discuss the potential challenges to the design of a computationally efficient learning algorithm with a small overhead.
\end{abstract}

\section{Introduction}
Consider the following real-world scenario: we would like to train a speech recognition model based on labeled examples collected from different users. For this particular application, a high \emph{average} accuracy over all users is far from satisfactory: a model that is correct on $99.9\%$ of the data may still go seriously wrong for a small yet non-negligible $0.1\%$ fraction of the users. Instead, a more desirable objective would be finding personalized speech recognition solutions that are accurate for \emph{every single user}.

There are two major challenges to achieving this goal, the first being user heterogeneity: a model trained exclusively for users with a particular accent may fail miserably for users from another region. This challenge hints that a successful learning algorithm should be adaptive: more samples need to be collected from users with atypical data distributions. Equally crucial is that a small fraction of the users are malicious (e.g., they are controlled by a competing corporation); these users intend to mislead the speech recognition model into generating inaccurate or even ludicrous outputs.

Motivated by these practical concerns, we propose the \emph{Robust Collaborative Learning} model and study from a theoretical perspective the complexity of learning in the presence of untrusted collaborators. In our model, a learning algorithm interacts with $n$ different users, each associated with a data distribution $\D_i$. As mentioned above, a successful learning algorithm should, ideally, find personalized classifiers $f_1, f_2, \ldots, f_n$ for different distributions, such that \[\err_{\D_i}(f_i) \triangleq \Pr_{x \sim \D_i}[f_i(x) \ne f^*(x)] < \eps\] holds for every $i \in [n]$, where $f^*(x)$ denotes the true label of sample $x$. Further complicating the situation is that the algorithm can only interact with the data distributions via the users, each of which is either \emph{truthful} or \emph{adversarial}. A truthful user always provides the learning algorithm with independent samples drawn from his distribution together with the correct labels, whereas the labeled samples collected from adversarial users are arbitrary.

In the presence of malicious users, it is clearly impossible to learn an accurate classifier for every single distribution: an adversary may choose to provide no information about his data distribution. Therefore, a more realistic objective is to satisfy all the truthful users, i.e., to learn $n$ classifiers $f_1, f_2, \ldots, f_n$ such that $\err_{\D_i}(f_i) < \eps$ holds for every truthful user $i$.

Na\"ively, one could ignore the prior knowledge that samples from truthful users are labeled by the same function, and run $n$ independent copies of the same learning algorithm for the $n$ users. This straightforward approach clearly needs at least $n$ times as many samples as that required by learning on a single data distribution. Following the terminology of~\citet{blum2017collaborative}, we say that this na\"ive algorithm leads to an $\Omega(n)$ sample complexity \emph{overhead}. The notion of overhead measures the extent to which learning benefits from the collaboration and sharing of information among different parties. \citet{blum2017collaborative} proposed a learning algorithm that achieves an $O(\ln n)$ overhead for the case that all users are truthful, i.e., $\eta = 0$. We are then interested in answering the following natural question: can we still achieve a sublinear overhead for the case that $\eta > 0$, at least when $\eta$ is sufficiently small? In other words, \emph{do adversaries ruin the efficiency of collaboration?}

\subsection{Model and Preliminaries}
Similar to the classic \emph{Probably Approximately Correct (PAC) learning} framework due to \citet{valiant1984theory}, we consider the binary classification problem on a set $\X$. The hypothesis class $\F$ is a collection of binary functions on $\X$ with VC-dimension $d$. The elements in $\X$ are labeled by an unknown target function $f^* \in \F$.\footnote{This is known as the \emph{realizable} setting of PAC learning.}

Suppose that $\D$ is a probability distribution on set $\X$. Let $\O_{\F}$ denote the oracle that, given a set $S = \{(x_i, y_i)\}$ of labeled examples, either returns a classifier $f \in \F$ that is consistent with the examples (i.e., $f(x_i) = y_i$ for every $(x_i, y_i) \in S$) or returns $\bot$ if $\F$ contains no such consistent classifiers. A classic result in PAC learning states that if \[m = \Theta\left(\frac{d\ln(1 / \eps) + \ln(1 / \delta)}{\eps}\right)\] independent labeled samples $S = \{(x_i, f^*(x_i)): i \in [m]\}$ are drawn from $\D$, with probability at least $1 - \delta$, inequality $\err_{\D}(f) < \eps$ holds for every possible output $f = \O_{\F}\left(S\right)$~\cite{blumer1989learnability}.

In the Robust Collaborative Learning setting, we consider $n$ different data distributions $\D_1, \D_2, \ldots, \D_n$ supported on $\X$. A learning algorithm interacts with these distributions via $n$ \emph{user oracles} $\O_1, \O_2, \ldots, \O_n$, each of which operates in one of two different modes: \emph{truthful} or \emph{adversarial}. Upon each call to a truthful oracle $\O_i$, a sample $x$ is drawn from distribution $\D_i$ and the labeled sample $(x, f^*(x))$ is returned. On the other hand, an adversarial oracle $\O_i$ may output an arbitrary pair in $\X \times \{0, 1\}$ each time.\footnote{Our results hold even if the adversarial oracles are allowed to collude and they know the samples previously drawn by truthful oracles.}

We define $(\eps, \delta, \eta)$-learning in the Robust Collaborative Learning model as the task of learning an $\eps$-accurate classifier for each truthful user with probability $1 - \delta$, under the assumption that at most an $\eta$ fraction of the oracles are adversarial.
\begin{definition}[$(\eps, \delta, \eta)$-learning]
	Algorithm $\A$ is an $(\eps, \delta, \eta)$-learning algorithm if $\A$, given a concept class $\F$ and access to $n$ user oracles $\O_1, \O_2, \ldots, \O_n$ among which at most $\eta n$ oracles are adversarial, outputs functions $f_1, f_2, \ldots, f_n: \X \to \{0, 1\}$, such that with probability at least $1 - \delta$, $\err_{\D_i}(f_i) < \eps$ holds simultaneously for every truthful oracle $\O_i$.
\end{definition}

We also formally define the sample complexity of $(\eps, \delta, \eta)$-learning.
\begin{definition}[Sample Complexity]
	Let $M_{\A}(\F, \{\O_i\})$ denote the expected number of times that algorithm $\A$ calls oracles $\O_1, \O_2, \ldots, \O_n$ in total, when it runs on hypothesis class $\F$ and user oracles $\{\O_i\}$. The sample complexity of $(\eps, \delta, \eta)$-learning a concept class with VC-dimension $d$ from $n$ users is defined as:
    \[m_{n, d}(\eps, \delta, \eta) \triangleq \inf_{\A}\sup_{\F, \{\O_i\}}M_{\A}\left(\F, \{\O_i\}\right).\]
Here the infimum is over all $(\eps, \delta, \eta)$-learning algorithms $\A$. The supremum is taken over all hypothesis classes $\F$ with VC-dimension $d$ and user oracles $\O_1, \O_2, \ldots, \O_n$, among which at most an $\eta$ fraction are adversarial. 
\end{definition}

The \emph{overhead} of Robust Collaborative Learning is defined as the ratio between the sample complexity $m_{n, d}(\eps, \delta, \eta)$ and its counterpart in the classic PAC learning setting, $m_{1, d}(\eps, \delta, 0)$. To simplify the notations and restrict our attention to the dependence of overhead on parameters $n$, $d$ and $\eta$, we assume that $\eps=\delta=0.1$ in our definition of overhead.\footnote{This definition only changes by a constant factor when $0.1$ is replaced by other sufficiently small constants.}

\begin{definition}[Overhead]
For $n, d \in \N$ and $\eta \in [0, 1]$, the sample complexity overhead of Robust Collaborative Learning is defined as \[o(n, d, \eta) \triangleq \frac{m_{n, d}(\eps, \delta, \eta)}{m_{1, d}(\eps, \delta, 0)},\] where $\eps = \delta = 0.1$.
\end{definition}

Following our definition of the overhead, the results in \cite{blum2017collaborative} imply that when all users are truthful (i.e., when $\eta = 0$) and $n = O(d)$, $o(n, d, 0) = O(\ln n)$. They also proved the tightness of this bound in the special case that $n = \Theta(d)$.

\subsection{Our Results}
\paragraph{Information-theoretically, collaboration can be robust.}
In Section~\ref{sec:upper}, we present our main positive result: a learning algorithm that achieves an $O(\eta n + \ln n)$ sample complexity overhead when $n = O(d)$. Our result recovers the $O(\ln n)$ overhead upper bound due to~\citet{blum2017collaborative} for the special case $\eta = 0$. In Section~\ref{sec:lower}, we complement our positive result with a lower bound, which states that an $\Omega(\eta n)$ overhead is inevitable in the worst case. In light of the previous $\Omega(\ln n)$ overhead lower bound for the special case that $n = \Theta(d)$~\cite{blum2017collaborative}, our learning algorithm achieves an optimal overhead when parameters $n$ and $d$ differ by a bounded constant factor.

Our characterization of the sample complexity in Robust Collaborative Learning indicates that efficient cooperation is possible even if a small fraction of arbitrary outliers are present. Moreover, the overhead is largely determined by $\eta n$, the maximum possible number of adversaries. Our results suggest that for practical applications, the learning algorithm could greatly benefit from a relatively clean pool of data sources.

\paragraph{Computationally, outliers may ruin collaboration.}
Our study focuses on the sample complexity of Robust Collaborative Learning, yet also important in practice is the amount of computational power required by the learning task. Indeed, the algorithm that we propose in Section~\ref{sec:upper} is inefficient due to an exhaustive enumeration of the set of truthful users, which takes exponential time. In Section~\ref{sec:discussion}, we provide evidence that hints at a time-sample complexity tradeoff in Robust Collaborative Learning. Informally, we conjecture that any learning algorithm with a sublinear overhead must run in super-polynomial time. In other words, while the presence of adversaries does not seriously increase the sample complexity of learning, it may still ruin the efficiency of collaboration by significantly increasing the computational burden of this learning task. We support our conjecture with known hardness results in computational complexity theory.

\section{Related Work}\label{sec:related}
Most related to our work is the recent \emph{Collaborative PAC Learning} model proposed by~\citet{blum2017collaborative}. They also considered the task of learning the same binary classifier on different data distributions, yet all users are assumed to be truthful in their model. In fact, the Robust Collaborative Learning model reduces to the \emph{personalized setting} of their model when $\eta=0$. Here the word ``personalized'' emphasizes the assumption that each user may receive a specialized classifier tailored to his distribution.

In addition to the personalized setting, they also studied the \emph{centralized setting}, in which all the $n$ users should receive the same classifier. They proved that a poly-logarithmic overhead is still achievable in this more challenging setting. In our Robust Collaborative Learning model, however, centralized learning is in general impossible due to the indistinguishability between truthful and adversarial users. The following simple impossibility result holds for extremely simple concept classes and even when infinitely many samples are available.

\begin{proposition}\label{prop:central}
For any $\eps \in [0, 1)$, $\delta \in \left[0, \frac{1}{2}\right)$ and $\eta \in (0, 1]$, no algorithms $(\eps, \delta, \eta)$-learn any concept class of VC-dimension $d\ge2$, under the restriction that all users should receive the same classifier.
\end{proposition}

\begin{proof}[Proof of Proposition~\ref{prop:central}]
	Let $x_0$ and $x_1$ be two different samples that can be shattered by $\F$. Choose $f_0, f_1\in\F$ such that $f_0(x_0) = f_0(x_1) = f_1(x_0) = 0$ and $f_1(x_1) = 1$. Let $n$ be large enough such that $\eta n\ge1$. Construct (degenerate) distributions $\D_1, \D_2, \ldots, \D_n$ such that $\D_1(x_1)=\D_2(x_1)=1$ and $\D_i(x_0)=1$ for each $3\le i\le n$.

	Consider the following two cases:
	\begin{itemize}
	\item The target function is $f_0$. The only adversarial user, $\O_1$, misleads the learning algorithm by outputing the labeled example $(x_1, 1)$.
	\item The target function is $f_1$. The only adversarial user, $\O_2$, misleads the learning algorithm by outputing the labeled example $(x_1, 0)$.
	\end{itemize}

	Note that in both cases, oracles $\O_1$ and $\O_2$ always return $(x_1, 1)$ and $(x_1, 0)$ respectively, while all other oracles return $(x_0, 0)$. Consequently, no algorithms can distinguish these two cases with success probability strictly greater than $\frac{1}{2}$. Thus, any learning algorithm would have a failure probability of at least $\frac{1}{2} > \delta$.
\end{proof}

A related line of research is multi-task learning~\cite{caruana1997multitask,baxter2000model,ben2002theoretical,ben2003exploiting}, which studies the problem of learning multiple related tasks simultaneously with significantly fewer samples. Most work in this direction assumes certain relation (e.g., a transfer function) between the given learning tasks. In contrast to multi-task learning, our work focuses on the problem of learning the same classifier on multiple data distributions, without assuming any similarity between these underlying distributions.

Also relevant to our study is the work on robust statistics, i.e., the study of learning and estimation in the presence of noisy data and arbitrary outliers; see~\citet{lai2016agnostic,charikar2017learning,diakonikolas2016robust,diakonikolas2017being,diakonikolas2018robustly} and the references therein for some recent work in this line of research. Classic problems in this regime include the estimation of the mean and covariance of a high-dimensional distribution, given a dataset consisting of samples drawn from the distribution and a small fraction of arbitrary outliers. Our model differs from this line of research in that we consider a general classification setting, and the learning algorithm is allowed to sample different sources adaptively, instead of learning from a given dataset of fixed size.
\section{An Iterative Learning Algorithm}\label{sec:upper}
In this section, we present an iterative $(\eps, \delta, \eta)$-learning algorithm achieves an $O(\eta n + \ln n)$ overhead when $n = O(d)$. Here $n$ is the number of users, and $d$ denotes the VC-dimension of the hypothesis class $\F$. Since $\F$ can be large and even infinite, we assume that the algorithm access $\F$ via an oracle $\O_{\F}$ that, given a set $S = \{(x_i, y_i)\}$ of labeled examples, either returns a classifier $f \in \F$ such that $f(x_i) = y_i$ holds for each pair $(x_i, y_i) \in S$, or returns $\bot$ if $\F$ does not contain any consistent functions. The algorithm interacts with the underlying data distributions $\D_1, \D_2, \ldots, \D_n$ via $n$ example oracles $\O_1, \O_2, \ldots, \O_n$, among which at most an $\eta$ fraction are adversarial.

\subsection{Algorithm}

Our algorithm is formally described in Algorithms \ref{alg:main}~through~\ref{alg:test}. The main algorithm proceeds in rounds and maintains a set $G_r$ of the indices of \emph{active users} at the beginning of round $r$, i.e., users who have not received an $\eps$-accurate classifier so far. When $\lfloor\eta n\rfloor$, the maximum possible number of adversaries, is below $\frac{|G_r|}{10}$, the algorithm invokes subroutine $\Cand$ to find a \emph{candidate classifier} $\hat f_r$. Then, Algorithm~\ref{alg:main} calls the validation procedure $\Test$ to check whether $\hat f_r$ is accurate for each user $i \in G_r$ (with respect to accuracy threshold $\eps$). If so, the algorithm marks the output for user $i$ as $\hat f_r$; otherwise, user $i$ stays in set $G_{r + 1}$ for the next round. When the proportion of adversaries reaches $\frac{1}{10}$, the algorithm learns for the remaining users independently: for each active user, it draws samples from his oracle and outputs an arbitrary classifier that is consistent with his data.

\begin{algorithm}[H]
	\caption{Iterative Robust Collaborative Learning}
	\label{alg:main}
	\KwIn{Parameters $n$, $d$, $\eps$, $\delta$, $\eta$.}
	\KwOut{Classifiers $f_1, f_2, \ldots, f_n$.}
	$r \gets 1$; $G_1 \gets [n]$\;
	\While{$\lfloor\eta n\rfloor \le \frac{|G_r|}{10}$} {
		$\delta_r \gets \frac{\delta}{5r^2}$\;
		$\hat f_r \gets \Cand(G_r, d, \eps, \delta_r)$\;
		$G_{r + 1} \gets \Test(G_r, \hat f_r, \eps, \delta_r)$\;
		Set $f_i \gets \hat f_r$ for each $i \in G_r \setminus G_{r + 1}$\;
		$r \gets r + 1$\;
	}
	\For{$i \in G_r$} {
		$S_i \gets$ $\Theta\left(\frac{d \ln(1/\eps) + \ln(n/\delta)}{\eps}\right)$ labeled samples from $\O_i$\;
		$f_i \gets \O_{\F}(S_i)$\;
	}
	\Return $f_1, f_2, \ldots, f_n$\;
\end{algorithm}

\begin{algorithm}[H]
	\caption{$\Cand(G, d, \eps, \delta)$}
	\label{alg:cand}
	\KwIn{Index set $G$, parameters $d$, $\eps$ and $\delta$.}
	\KwOut{Candidate classifier $\hat f \in \F$.}
	$M \gets \Theta\left(\frac{d\ln(1/\eps) + \ln\left(2^{|G|}/\delta\right)}{\eps} + |G|\ln\frac{|G|}{\delta}\right)$\;
	\For{$i \in G$} {
		$S_i$ $\gets$ $\frac{4M}{|G|}$ labeled samples from $\O_i$\;
	}
	$\G \gets \left\{H \subseteq G: |H| \ge \frac{9}{10}|G|\right\}$\;
	\For{$H \in \G$} {
		$\hat f_H \gets \O_{\F}(\bigcup_{i \in H}S_i)$\;
		\lIf{$\hat f_H \ne \bot$} {\Return $\hat f_H$}
	}
\end{algorithm}

\begin{algorithm}[H]
	\caption{$\Test(G, \hat f, \eps, \delta)$}
	\label{alg:test}
	\KwIn{Set $G$ of indices, candidate function $\hat f$, parameters $\eps$ and $\delta$.}
	\KwOut{Set $G'$ of surviving indices.}
	\For{$i \in G$} {
		$S_i \gets$ $\Theta\Big(\frac{\ln(|G| / \delta)}{\eps}\Big)$ samples from $\O_i$\;
		$\theta_i \gets \frac{1}{|S_i|}\sum_{(x, y)\in S_i}\Ind{\hat f(x) \ne y}$\;
	}
	\Return $G' = \{i \in G: \theta_i > \frac{3}{4}\eps\}$\;
\end{algorithm}

\subsection{Analysis of Subroutines}
Subroutine $\Cand$ (Algorithm~\ref{alg:cand}) is the key to the sample efficiency of our algorithm, as it enables us to learn a candidate classifier that is accurate simultaneously for a constant fraction of the active users, using only a nearly-linear number of samples (with respect to parameters $|G|$ and $d$). Subroutine $\Test$ (Algorithm~\ref{alg:test}) further checks whether the learned classifier is accurate enough for each active user. This allows us to determine whether a user should remain active in the next iteration. We devote this subsection to the analysis of these two subroutines.

\begin{lemma}\label{lem:half}
Suppose $G \subseteq [n]$ denotes the indices of $|G|$ users, among which at most $\frac{|G|}{10}$ are adversarial. Let $\hat f$ denote the output of $\Cand(G, d, \eps, \delta)$. With probability $1 - \delta$, the following two conditions hold simultaneously for at least $\frac{|G|}{2}$ indices $i \in G$: (1) $\err_{\D_i}(\hat f) \le \frac{\eps}{2}$; (2) oracle $\O_i$ is truthful.
\end{lemma}

The proof of Lemma~\ref{lem:half} relies on the following technical claim, which enables us to relate the union of several equal-size datasets to the samples drawn from the uniform mixture of the corresponding distributions.

\begin{claim}\label{claim:bins}
	Suppose $m = \Omega\left(n \ln\frac{n}{\delta}\right)$ balls are thrown into $n$ bins independently and uniformly at random. Then with probability $1 - \delta$, no bins contain more than $\frac{2m}{n}$ balls.
\end{claim}

\begin{proof}[Proof of Claim~\ref{claim:bins}]
	Let random variable $X$ denote the number of balls in a fixed bin, so $\frac{X}{m}$ is the average of $m$ i.i.d. Bernoulli random variables with mean $\frac{1}{n}$. The Chernoff bound implies that \[\Pr\left[\frac{X}{m} \ge \frac{2}{n}\right] \le e^{-mD\left(\frac{2}{n}\right|\left|\frac{1}{n}\right)} = e^{-\Omega(n\ln\frac{n}{\delta})\cdot\Omega(\frac{1}{n})} \le \frac{\delta}{n},\]
	where the last step holds if we choose a sufficiently large hidden constant in $m = \Omega\left(n\ln\frac{n}{\delta}\right)$. The claim follows from a union bound over the $n$ bins.
\end{proof}

\begin{proof}[Proof of Lemma~\ref{lem:half}]
Let $G'$ denote the indices of truthful users in $G$. By assumption, $|G'| \ge \frac{9}{10}|G|$ and $\F$ contains a function $f^*$ that is consistent with $\bigcup_{i \in G'}S_i$. This guarantees that Algorithm~\ref{alg:cand} should return $\hat f_{H}$ as the output when $H = G'$, so function $\Cand$ is well-defined.

Recall that in Algorithm~\ref{alg:cand}, we set \[M = \Theta\left(\frac{d\ln(1/\eps) + \ln\left(2^{|G|}/\delta\right)}{\eps} + |G|\ln\frac{|G|}{\delta}\right).\] Consider the following thought experiment. For each non-empty $H \subseteq G$, we draw a sequence $A_H$ of $M$ integers, each of which is chosen uniformly and independently at random from $H$. We also draw $M$ samples from oracle $\O_i$ for each $i \in G$. If all users in $H$ are truthful, the samples together with $A_H$ naturally specify a realization of drawing $M$ samples from the uniform mixture distribution $\D_H\triangleq\frac{1}{|H|}\sum_{i\in H}\D_i$: we arrange the $M$ samples drawn from each distribution into a queue, and when we would like to draw the $i$-th sample, we simply take the sample at the front of queue $A_H(i)$.

For a fixed non-empty subset $H \subseteq G$ that only contains truthful users, the VC theorem implies that with probability $1 - \frac{\delta}{2^{|G| + 1}}$ (over the randomness in both the samples and the choice of $A_H$), when we draw samples from the uniform mixture $\D_H$ as described above, any function $f \in \F$ that is consistent with the labeled samples satisfies $\err_{\D_H}(f) \le \frac{\eps}{10}$. By a union bound over $\le 2^{|G|}$ different sets $H \subseteq G$, the above holds for \emph{every} $H \subseteq G$ simultaneously with probability $1 - 2^{|G|} \cdot \frac{\delta}{2^{|G| + 1}} = 1 - \frac{\delta}{2}$.

Recall that in Algorithm~\ref{alg:cand}, we first query each oracle $\O_i$ to obtain a ``training set'' $S_i$ of size $\frac{4M}{|G|}$ for each $i \in G$. Then we find set $H \subseteq G$ and classifier $\hat f_H \in \F$ such that: (1) $|H| \ge \frac{9}{10}|G|$; (2) $\hat f_H$ is consistent with all labeled samples in $\bigcup_{i\in H}S_i$. Suppose that $H$ is the set associated with the output of Algorithm~\ref{alg:cand}, and let $H' = \{i\in H: \O_i\text{ is truthful}\}$. Note that $|H'| \ge |H| - \frac{|G|}{10} \ge \frac{4}{5}|G|$.

The crucial observation is that since \[M = \Omega\left(|G|\ln\frac{|G|}{\delta}\right),\] Claim~\ref{claim:bins} implies that with probability at least $1 - \frac{\delta}{2}$, each index $i \in H'$ appears less than $\frac{2M}{|H'|} \le \frac{4M}{|G|}$ times in $A_{H'}$. In other words, $\bigcup_{i\in H'}S_i$ is a superset of the $M$ samples that are supposed to be drawn from $\D_{H'}$ (in our thought experiment). Since $\hat f_H$ is consistent with $\bigcup_{i\in H'}S_i$, a union bound shows that with probability $1 - 2\cdot\frac{\delta}{2} = 1 - \delta$, we have \[\frac{1}{|H'|}\sum_{i\in H'}\err_{\D_i}(\hat f_H) = \err_{\D_{H'}}(\hat f_H)\le \frac{\eps}{10}.\] This further implies that $\err_{\D_i}(\hat f_H) \le \frac{\eps}{2}$ holds for at least $\frac{|G|}{2}$ indices $i \in H'$; otherwise, we would have
\begin{align*}
	\frac{1}{|H'|}\sum_{i \in H'}\err_{\D_i}(\hat f_H) &\ge \frac{1}{|H'|}\left(|H'| - \frac{|G|}{2}\right) \cdot \frac{\eps}{2}\\
	&\ge \left(1 - \frac{5}{8}\right) \cdot \frac{\eps}{2} > \frac{\eps}{10},
\end{align*}
which leads to a contradiction. Here the second step applies $|H'| \ge \frac{4}{5}|G|$. This proves the lemma.
\end{proof}

The following lemma, which directly follows from a Chernoff bound and a union bound, states that with probability $1 - \delta$, $\Test(G, \hat f, \eps, \delta)$ correctly determines whether $\hat f$ has an $O(\eps)$ error for each user in $G$.
\begin{lemma}\label{lem:test}
	Let $G'$ denote the output of $\Test(G, \hat f, \eps, \delta)$.
	With probability $1 - \delta$, the following holds for every $i \in G$ simultaneously: (1) if $\err_{\D_i}(\hat f) > \eps$, $i \in G'$; (2) if $\err_{\D_i}(\hat f) \le \frac{\eps}{2}$, $i \notin G'$.
\end{lemma}
\begin{proof}[Proof of Lemma~\ref{lem:test}]
	Fix a truthful oracle $\O_i$ with $i \in G$. Recall that Algorithm~\ref{alg:test} sets \[\theta_i = \frac{1}{|S_i|}\sum_{(x, y) \in S_i}\Ind{\hat f(x) \ne y}.\] Note that $\theta_i$ is the average of $\Omega\left(\frac{\ln(|G| / \delta)}{\eps}\right)$ independent Bernoulli random variables, each with mean $\err_{\D_i}(\hat f)$. Thus, the Chernoff bound implies that with probability $1 - \frac{\delta}{|G|}$, the following two conditions holds simultaneously: (1) if $\err_{\D_i}(\hat f) > \eps$, $\theta_i > \frac{3}{4}\eps$; (2) if $\err_{\D_i}(\hat f) \le \frac{\eps}{2}$, $\theta_i \le \frac{3}{4}\eps$. The lemma follows from a union bound over all $i \in G$.
\end{proof}

\subsection{Correctness and Sample Complexity}
Now we are ready to prove our main result.

\begin{theorem}\label{thm:upper}
For any $\eps, \delta \in (0, 1]$ and $\eta \in [0, 1]$, Algorithm~\ref{alg:main} is an $(\eps, \delta, \eta)$-learning algorithm and takes at most  \[O\left(\frac{d\ln(1/\eps)}{\eps}(\eta n + \ln n) + \frac{n\ln(n / \delta)}{\eps}\right)\] samples.
\end{theorem}

By Theorem~\ref{thm:upper}, the sample complexity $m_{n,d}(\eps,\delta,\eta)$ reduces to $O\left(d(\eta n + \ln n)\right)$ when $\eps$ and $\delta$ are constants and $n \le C \cdot d$ for some constant $C$. Therefore, when $n = O(d)$, we have the following overhead upper bound: \[o(n,d,\eta) = \frac{O\left(d(\eta n + \ln n)\right)}{\Theta(d)} = O(\eta n + \ln n).\]

\begin{proof}[Proof of Theorem~\ref{thm:upper}]
The proof proceeds by applying Lemmas \ref{lem:half}~and~\ref{lem:test} iteratively. In each round $r$, Lemma~\ref{lem:half} guarantees that with probability $1 - \delta_r$, the learned classifier $\hat f_r$ has an error below $\frac{\eps}{2}$ for at least $\frac{|G_r|}{2}$ truthful users. By Lemma~\ref{lem:test}, for each such distribution, the ``validation error'' $\theta_i$ should be below $\frac{3}{4}\eps$, so these users will exit the algorithm by receiving $\hat f_r$ as the classifier, and the number of active users decreases by a factor of $\frac{1}{2}$. Therefore, the while-loop in Algorithm~\ref{alg:main} terminates after at most $\lfloor\log_2n\rfloor + 1$ iterations. Finally, the algorithm satisfies the remaining active users by drawing $\Theta\left(\frac{d\ln(1/\eps)+\ln(n/\delta)}{\eps}\right)$ samples from each of them. Thus, the VC theorem guarantees that for each truthful user, the learned classifier is $\eps$-accurate with probability at least $1 - \frac{\delta}{3n}$. By a union bound, with probability at least \[1 - \sum_{r=1}^{\infty}2\delta_r - n\cdot\frac{\delta}{3n} = 1 - \left(\frac{1}{3} + \sum_{r=1}^{\infty}\frac{2}{5r^2}\right)\delta \ge 1 - \delta,\] Algorithm~\ref{alg:main} returns an $\eps$-accurate classifier for each truthful user.

It remains to bound the sample complexity of Algorithm~\ref{alg:main}. In round $r$, the number of active users is at most $|G_r| \le \frac{n}{2^{r-1}}$. Recall that $\delta_r = \frac{\delta}{5r^2}$. The number of samples that $\Cand$ and $\Test$ draw in round $r$ is then upper bounded by
\begin{align*}
	&O\left(\frac{d \ln(1 / \eps)+ |G_r|\ln(|G_r| / \delta_r)}{\eps}\right)\\
=	&O\left(\frac{d \ln(1 / \eps)}{\eps} + \frac{n\ln(n / \delta)}{2^r\eps}\right).
\end{align*}
Therefore, the number of samples drawn in the $O(\ln n)$ iterations is upper bounded by:
\begin{equation}\label{eq:upper-1}\begin{split}
	&\sum_{r = 0}^{\lfloor\log_2 n\rfloor+1}O\left(\frac{d \ln(1 / \eps)}{\eps} + \frac{n \ln(n / \delta)}{2^r\eps}\right)\\
=	&O\left(\frac{d\ln(1 / \eps)\ln n+n \ln(n / \delta)}{\eps}\right).
\end{split}\end{equation}

When the while-loop in Algorithm~\ref{alg:main} terminates, it holds that $|G_r| \le 10\eta n = O(\eta n)$. After that, we learn on the remaining distributions separately, using
\begin{equation}\label{eq:upper-2}
O\left(\eta n \cdot \frac{d\ln(1/\eps)+\ln(n/\delta)}{\eps}\right)
\end{equation}
samples in total. Adding \eqref{eq:upper-1} and \eqref{eq:upper-2} gives the desired sample complexity upper bound.
\end{proof}

\section{Overhead Lower Bound}\label{sec:lower}
In this section, we show that an $\Omega(\eta n+\ln n)$ overhead is unavoidable when $n = \Theta(d)$. Therefore, the overhead achieved by Algorithm~\ref{alg:main} is optimal up to a constant factor, when the number of users is commensurate with the complexity of the hypothesis class. Formally, we have the following theorem.

\begin{theorem}\label{thm:lower}
For any $n, d \in \N$, $\eps \in \left(0, \frac{1}{2}\right]$ and $\delta, \eta \in (0, 1)$, \[m_{n, d}(\eps, \delta, \eta) = \Omega\left(\frac{\eta nd}{\eps}\right).\]
\end{theorem}

Theorem~\ref{thm:lower} directly implies the following lower bound on the overhead: \[o(n, d, \eta) = \frac{\Omega(\eta n d)}{\Theta(d)} = \Omega(\eta n).\] Combining this with the previous lower bound $o(n, d, \eta) = \Omega(\ln n)$ when $n = \Theta(d)$ and $\eta = 0$~\cite{blum2017collaborative}\footnote{They proved an $\Omega(\ln n)$ lower bound for the special case that $n = d$, yet their proof directly implies the same lower bound when $n = \Theta(d)$.}, we obtain the desired worst-case lower bound of $\Omega(\eta n+\ln n)$.

\begin{proof}[Proof of Theorem~\ref{thm:lower}]
Assume without loss of generality that $\eta n$ is an integer between $1$ and $n - 1$. We consider the binary classification problem on set $\X = [d]\cup\{\bot\}$, while the hypothesis class $\F$ contains all the $2^d$ binary functions on $\X$ that map $\bot$ to $0$. The target function $f^*$ is chosen uniformly at random from $\F$.

Suppose that for $(1-\eta)n-1$ truthful users, the data distribution is the degenerate distribution on $\{\bot\}$, so these truthful users provide no information on the correct classifier $f^*$. On the other hand, the data distribution of the only remaining truthful user $i^*$ satisfies $\D_{i^*}(x) = \frac{2\eps}{d}$ for any $x \in [d]$ and $\D_{i^*}(\bot) = 1 - 2\eps$. By construction, a learning algorithm must draw $\Omega\left(\frac{d}{\eps}\right)$ samples from $\D_{i^*}$ in order to learn an $\eps$-accurate classifier with a non-trivial success probability $1 - \delta$.

Now suppose that each of the $\eta n$ adversarial users tries to pretend that he is the truthful user $i^*$. More specifically, each adversarial user $i$ chooses a function $\tilde f_i \in \F$ uniformly at random, and answer the queries as if he is the truthful user with a different target function $\tilde f_i$. In other words, upon each request from the learning algorithm, oracle $\O_i$ draws $x$ from $\D_{i^*}$ and returns $\left(x, \tilde f_i(x)\right)$.

Recall that the actual target function $f^*$ is also uniformly distributed in $\F$, so from the perspective of the learning algorithm, the truthful user $i^*$ is indistinguishable from the other $\eta n$ adversarial users. Therefore, an $(\eps, \delta, \eta)$-learning algorithm must guarantee that each of these $(\eta n + 1)$ users receives an $\eps$-accurate classifier with probability at least $1 - \delta$. The sample complexity lower bound from PAC learning theory implies that we must draw $\Omega\left(\frac{d}{\eps}\right)$ samples from each such user and thus \[(\eta n + 1)\cdot\Omega\left(\frac{d}{\eps}\right) = \Omega\left(\frac{\eta nd}{\eps}\right)\] samples in total.
\end{proof}
\section{Discussion: A Computationally Efficient Algorithm?}\label{sec:discussion}
Although Algorithm~\ref{alg:main} is proved to achieve an optimal sample complexity overhead in certain cases, the algorithm is computationally inefficient and of limited practical use when there are a large number of users. In particular, subroutine $\Cand$ performs an exhaustive search over all user subsets of size $\ge \frac{9}{10}|G|$, and thus may potentially call oracle $\O_{\F}$ exponentially many times. In contrast, the na\"ive approach that learns for different users separately, though obtaining an $\Omega(n)$ overhead, only makes $n$ calls to oracle $\O_{\F}$. Naturally, one may wonder whether we can achieve the best of both worlds by finding a computationally efficient learning algorithm with a small overhead? We conjecture that such an algorithm, unfortunately, does not exist.

\begin{conjecture}\label{conj:tradeoff}
	For any $\alpha > -1$ and $\beta < 1$, no learning algorithms that make polynomially many calls to oracle $\O_{\F}$ achieve an $O(n^{\beta})$ overhead when $\eta = \Omega(n^{\alpha})$.
\end{conjecture}

In words, when there is a non-trivial number of adversaries, any efficient learning algorithm would incur a nearly-linear overhead. We remark that it is necessary to assume $\alpha > -1$ since when $\eta n$, the maximum possible number of adversaries, is a constant, the learning algorithm can enumerate the subset of adversarial users in polynomial time, thus achieving an optimal overhead efficiently. Proving or refuting Conjecture~\ref{conj:tradeoff} would greatly further our understanding of the impact of arbitrary outliers on collaborative learning.

The key to our sample-efficient learning algorithm is that subroutine $\Cand$ identifies a large user group such that some classifier $\hat f \in \F$ is consistent with all their labeled samples. Lemma~\ref{lem:half} further guarantees that $\hat f$ is $\eps$-accurate for at least half of the users. This allows us to satisfy almost all the users in $O(\ln n)$ iterations, resulting in the $\ln n$ term in the overhead.

We note that finding a group of users with consistent datasets generalizes the problem of finding a large clique in a graph: For an undirected graph with vertices labeled from $1$ to $n$, we construct the user oracles $\O_1, \O_2, \ldots, \O_n$ such that $\O_i$ and $\O_j$ produce conflicting labels on the same data if the edge $(i, j)$ is absent from the graph. Then a group of users have consistent datasets if and only if they form a clique in the corresponding graph.

Unfortunately, \citet{zuckerman2006linear} proved that even if the graph is known to contain a hidden clique of size $\Omega(n)$\footnote{Analogously, in our setting, we know that a large fraction of users have non-conflicting datasets.}, it is still NP-hard to find a clique of size $\Omega(n^{1 - \beta})$ for any $\beta < 1$. This indicates that, following the approach of Algorithm~\ref{alg:main}, a computationally efficient algorithm can only find accurate classifiers for at most $O(n^{1 - \beta})$ users in each iteration. As a result, $\Omega(n^{\beta})$ iterations would be necessary to satisfy all the $n$ users. The algorithm consequently incurs an $\Omega(n^{\beta})$ overhead.

\bibliographystyle{plainnat}
\bibliography{main}

\end{document}